\pdfoutput=1

\documentclass[11pt]{article}

\usepackage{emnlp2021}

\usepackage{times}
\usepackage{latexsym}

\usepackage{amsmath}
\usepackage{amsfonts}
\usepackage{amsthm}
\usepackage{bbm}
\usepackage{multirow}

\usepackage{microtype}
\usepackage{tipa}
\usepackage{inconsolata}
\usepackage{enumerate}

\usepackage{caption}
\usepackage{subcaption}
\usepackage{pifont}%

\usepackage{graphicx} 

\newcommand{\citeposs}[1]{\citeauthor{#1}'s (\citeyear{#1})}

\usepackage{booktabs}
\usepackage{cleveref}
\crefname{section}{\S}{\S\S}
\Crefname{section}{\S}{\S\S}
\crefname{table}{Table}{Tables}
\crefname{figure}{Figure}{Figures}
\crefname{algorithm}{Algorithm}{}
\crefname{equation}{eq.}{}
\crefname{appendix}{App.}{}
\crefname{prop}{Proposition}{}
\crefformat{section}{\S#2#1#3}  %

\newtheorem{theorem}{Theorem}[section]

\newtheorem{lemma}[theorem]{Lemma}

\newcommand{\bw}{\mathbf{w}}
\newcommand{\calW}{\mathcal{W}}

\newcommand{\bh}{\mathbf{h}}
\newcommand{\bz}{\mathbf{z}}
\newcommand{\bW}{\mathbf{W}}

\renewcommand{\hat}{\widehat}

\newcommand{\renyi}{R\'{e}nyi}
\newcommand{\ent}{\mathrm{H}}
\newcommand{\samplerenyi}{\mathrm{R}}

\newcommand{\defn}[1]{\textbf{#1}}
\newcommand{\ipaword}[1]{/\textipa{#1}/}
\newcommand{\sense}[1]{\textsc{#1}}

\title{On Homophony and R{\'e}nyi Entropy}

\usepackage{emoji}

\newcommand{\night}{\emoji[openmoji]{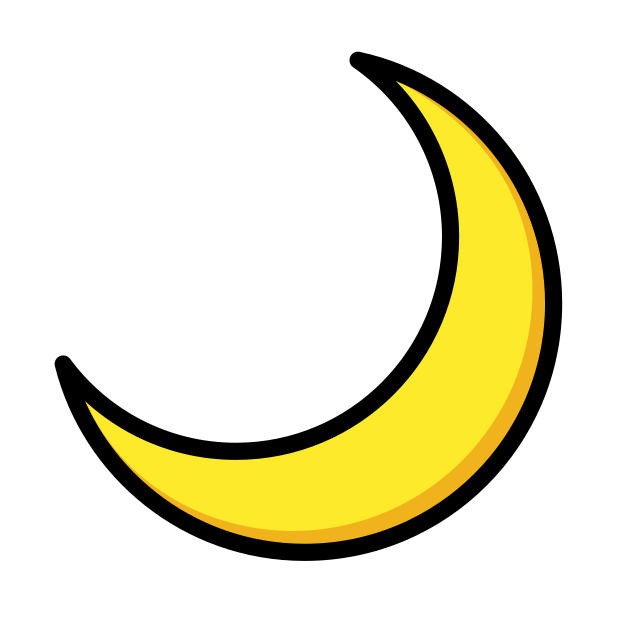}}

\newcommand{\knight}{\emoji[openmoji]{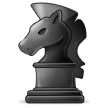}}

\newcommand{\emo}[1]{\raise1.0ex\hbox{\normalfont#1}}

\newcommand{\ucambridge}{\emo{\knight}}
\newcommand{\ethz}{\emo{\night}}

\author{Tiago Pimentel\ucambridge{}~\;~ Clara Meister\ethz{}~\;~ Simone Teufel\ucambridge{}~\;~ Ryan Cotterell\ucambridge{}\ethz{} \\
  \ucambridge{} University of Cambridge~\;~\;~\;~\ethz{} ETH Z\"{u}rich \phantom{\ucambridge{}} \\
  \texttt{\href{mailto:tp472@cam.ac.uk}{tp472@cam.ac.uk}}~\;~
  \texttt{\href{mailto:clara.meister@inf.ethz.ch}{clara.meister@inf.ethz.ch}}~\;~ \\
  \texttt{\href{mailto:sht25@cl.cam.ac.uk}{sht25@cl.cam.ac.uk}}~\;~ \texttt{\href{mailto:ryan.cotterell@inf.ethz.ch}{ryan.cotterell@inf.ethz.ch}}
}

\date{}

\begin{document}
\maketitle
\begin{abstract}
Homophony's widespread presence in natural languages is a controversial topic.
Recent theories of language optimality have tried to justify its prevalence, despite its negative effects on cognitive processing time; e.g., \citet{piantadosi2012communicative} argued homophony enables the reuse of efficient wordforms and is thus beneficial for languages.
This hypothesis has recently been challenged by \citet{trott2020human}, who posit that good wordforms are more often homophonous simply because they are more phonotactically probable.
In this paper, we join in on the debate.
We first propose a new information-theoretic quantification of a language's homophony: the sample R\'{e}nyi entropy.
Then, we use this quantification to revisit \citeauthor{trott2020human}'s claims.
While their point is theoretically sound, a specific methodological issue in their experiments raises doubts about their results.
After addressing this issue, we find no clear pressure either towards or against homophony---a much more nuanced result than either \citeauthor{piantadosi2012communicative}'s or \citeauthor{trott2020human}'s findings.
\end{abstract}

\section{Introduction}
Ambiguity is a hallmark of human language, and is present at all levels of linguistic structure. Both the causes and resulting effects of ambiguity are topics which have sparked much debate. For example, while some claim it has a beneficial impact on communication, others have taken it as a sign of inefficiency.
In this work, we contribute to the debate surrounding a specific form of lexical ambiguity in which a wordform shares multiple unrelated meanings:  \defn{homophony}.\footnote{In English, one example is the word \ipaword{naIt} which, out of context, could mean either \sense{knight} or \sense{night}.}\looseness=-1

While the quantitative study of homophony dates back to (at least) \citet{zipf1949human}, recently, \newcite{trott2020human} proposed a new explanation for the higher rate of homophony amongst good, i.e. short or phonotactically well-formed, wordforms: if words were sampled i.i.d. from a phonotactic distribution, then good wordforms would simply be sampled more often.
This implies there is no pressure \emph{favouring} homophony in these words for the sake of efficiency, which directly opposes \citeposs{piantadosi2012communicative} hypothesis.
In fact, \citeauthor{trott2020human} go further, relying on their experimental results to argue that ``homophony may even be selected \emph{against} in real languages.''\footnote{Their argument is actually more subtle than this. They posit, for instance, that this pressure may be indirect or that there might not actually be a pressure against the existence of homophony per se, but rather that their results could reflect a constraint on the extent to which any given wordform can be saturated with distinct, unrelated meanings.}

In this work, we join this debate by proposing a novel quantification of a language's homophony: the \defn{sample \renyi{} entropy}\footnote{The \renyi{} entropy \citep{renyi1961measures} is a generalisation of Shannon's entropy \citep{shannon1948mathematical}.}---defined as the negative log likelihood (or surprisal) that two instances in an $M$-sized sample take on the same value. 
When measured on an observed lexicon, this measure holistically captures the chance that two wordforms coincide, i.e., that they are homophones, providing a new means to test whether lexicons have a pressure \emph{in favour} or \emph{against} homophony.\looseness=-1

Further, we revisit \citeauthor{trott2020human}'s claims. 
Whilst their theoretical arguments are sound, we believe their experimental design could not have provided concrete evidence for or against their hypothesis. 
Specifically, their inadequate modelling of the phonotactic distribution, through the use of weakly regularised $n$-grams, cause us to question conclusions drawn from their experiments.
We take measures to address this issue---relying on more expressive LSTM language models---and provide our own analysis of homophony in natural language.

Experimentally, we arrive at more nuanced results than prior work, finding no pressure either towards or against homophony.
We conclude with the warning that biases in our models---and those of other works---need to be deeply considered when relying on them to answer linguistics questions.

\section{Homophony}
Homophony is a widespread phenomenon which has long puzzled linguists.
On average, roughly $4\%$ of the words in a language are estimated to be homophones \citep{dautriche2015weaving}.
This rate, however, has a large variation across languages---in English, for instance, \citet{rodd2002making} estimates it to be $7.4\%$.  
A number of works hint at the inefficiencies homophony leads to: \citet{rodd2002making} find that homophonous words are recognised more slowly; \citet{mazzocco1997children} shows homophones are harder for children to learn.

Yet a large body of work has argued \emph{for} the efficiency of homophony. \newcite{piantadosi2012communicative} suggest ambiguity is a desirable property in that it  increases a language's communicative efficiency. 
Ambiguity would allow a language to reuse good wordforms, which falls in line with \citeauthor{zipf1949human}'s principle of least effort.
With this in mind, \citeauthor{piantadosi2012communicative} showed that short, frequent and phonotactically probable wordforms have more homophones than their counterparts. In further support of this hypothesis, \citet{dautriche2018learning} found children easily learn to differentiate homophones when pairs have distinct syntactic categories (and that homophony is more likely in these cases); \citet{pimentel-etal-2020-speakers} showed speakers make contexts more informative in the presence of lexical ambiguity. These results suggest people naturally navigate ambiguity.

However, \citeauthor{trott2020human} recently proposed a new explanation of \citeauthor{piantadosi2012communicative}'s findings: They attribute homophony to chance.
Specifically, they claim that if we model the phonotactic distribution probabilistically, we can see that more homophones amongst good wordforms should be \emph{expected}---they are simply more probable.
Yet their methodology to support this claim---comparing natural lexicons with artificial ones sampled from $n$-gram language models---has an important setback:
their use of 5-gram models with only weak Laplace smoothing.
These models are prone to overfitting. As such, it is not surprising that an artificially generated lexicon would contain more homophony than natural ones; as we will show, overfit distributions will likely produce more collisions.
Thus, it is not entirely clear what we can conclude from their experiments.\looseness=-1
\section{Quantifying Homophony}

As per its definition, homophony should be tightly related to a language's phonotactics---its distribution over wordforms.
In this section, we first provide a definition of a language's phonotactic distribution. We then present both the \renyi{} collision entropy and the sample \renyi{} entropy as new measures of homophony.

\subsection{Phonotactics and Wordforms}

Formally, phonotactics defines a language's set of plausible wordforms.
Its classic exemplification, provided by \citet{chomsky1965some}, is that while the unattested wordform \emph{blick} would be plausible in English, \emph{*bnick} would not.
Under a probabilistic interpretation \citep{hayes2008maximum,gorman2013generative}, this can be re-stated as \emph{blick} having high phonotactic probability, while \emph{*bnick} has low phonotactic probability.
Notedly, a language's phonotactics highly constrains its set of possible wordforms \citep{dautriche2017words} and, cross-linguistically, the size of these sets seems to be roughly constant \citep{pimentel-etal-2020-phonotactic}.
Further, phonotactics has a tight relationship with word frequency; more phonotactically likely words are more frequent \citep{mahowald2018word}.\looseness=-1

We model the phonotactic distribution over possible wordforms as a language model:
\begin{equation}
    p(\bw) = \prod_{t=1}^{|\bw|} p(w_t \mid \bw_{<t})
\end{equation}
whose support is the infinite set $\calW$---defined here as the Kleene closure of a phonetic alphabet $\Sigma^*$, albeit where all $\bw \in \calW$ are padded with beginning-of- and end-of-word symbols. 
Under this definition, highly plausible wordforms would be assigned high probability, and vice-versa.

\subsection{Entropy as a Measure of Homophony} \label{sec:renyi_entropy}

The \renyi{} entropy is a generalisation of the more well-known Shannon entropy.
By its information-theoretic definition, a natural parallel can be drawn between \renyi{} entropy and homophony. 
Its general form is defined as
\begin{equation}\label{eq:renyi_general}
    \ent_{\alpha}(p) = \frac{1}{1-\alpha}\log\left(\sum_{\bw \in \calW} p(\bw)^\alpha\right)
\end{equation}
for $\alpha \geq 0, \neq 1 $. 
If we take the limit $\alpha \rightarrow 1$, we recover the Shannon entropy:
\begin{align}
    \ent_1(p) = -\sum_{\bw \in \calW} p(\bw) \log p(\bw)
\end{align}
which captures the inherent uncertainty in a distribution, i.e. the larger its value, the less predictable the outcome. 
The case of $\alpha=2$ yields the collision entropy (sometimes just termed the \renyi{} entropy)
\begin{equation} \label{eq:collision}
    \ent_2(p) = -\log\sum_{\bw \in \calW} p(\bw)^2 
\end{equation}
which, in our setting, is the negative log likelihood that two wordforms sampled i.i.d. from the same distribution are the same.
It thus provides a natural quantification of homophony in a language where the words are distributed i.i.d.\footnote{We note that the \renyi{} collision entropy measures a specific notion of homophony, one which is closely related to the average number of meanings per wordform.
By selecting other values for $\alpha$ in the \renyi{} entropy, one can capture different properties of the phonotactic distribution.
The \renyi{} min-entropy $\ent_{\infty}(p)$, for instance, is defined by a choice of $\alpha=\infty$ in \cref{eq:renyi_general} and measures the surprisal of the most probable wordform, being instead closely related to the maximum number of meanings per wordform.}\looseness=-1

Although both the collision and Shannon's entropies are measures of uncertainty, they capture distinct properties of the distribution.
Shannon's entropy represents the \emph{expected} surprisal of observing any specific wordform $\bw$, while the collision entropy computes the surprisal that a pair of words have identical forms, independent of which form.

\subsection{Measuring Collisions in a Lexicon}

In the previous section, we presented the \renyi{} collision entropy as measured on a specific phonotactic distribution.
The measure in \cref{eq:collision} has the unstated assumption that a pair of wordforms would be sampled i.i.d. from this distribution---i.e., 
the probability of a collision is $p(\bw)^2$, as opposed to  $p(\bw^{(1)})p(\bw^{(2)} \mid \bw^{(1)})$.
We do not, however, know if this i.i.d. assumption is valid for naturally occurring lexica.
We thus propose a new measure, termed the \defn{sample \renyi{} entropy}, which does not inherently encode an i.i.d. assumption.
Given an observed lexicon $\widetilde{\bW} = \{\widetilde{\bw}^{(m)} \}_{m=1}^{M}$ of size $M$, 
we directly measure the surprisal of two randomly selected words being a homophone as
\begin{align} \label{eq:collision_lexicon}
    \samplerenyi (&\widetilde{\bW} ) =\\
    &- \log \frac{\sum\limits_{m=1}^{M} \, \, \sum\limits_{m'=1, m'\neq m}^M \mathbbm{1} \left\{ \widetilde{\bw}^{(m)} = \widetilde{\bw}^{(m')} \right\}}{M (M-1)} \nonumber
\end{align}
In words, the above equation estimates the likelihood of a collision as the number of observed homophones over the number of possible collisions.

Notably, if words are sampled i.i.d.---i.e., if there is no pressure in favour or against homophony---the sample \renyi{} entropy goes to the actual \renyi{} entropy in \cref{eq:collision} as $M \rightarrow \infty$. In other words, under the i.i.d. assumption, \cref{eq:collision_lexicon} is a consistent estimator of \cref{eq:collision}.%

\subsection{A Tractable Estimate of \renyi{} Entropy}
Note that in our setting, it is impossible to exactly calculate the \renyi{} entropy $\ent_2(p)$, given that the support of $p$ (i.e., $\calW$) is infinite.
In this work, we estimate $\hat{\ent}_2(p)$ over a %
subset $\calW_\delta \subset \calW$:
\begin{equation} \label{eq:renyi_approx}
    \hat{\ent}_2(p) = -\log\sum_{\bw \in \calW_\delta} p(\bw)^2 
\end{equation}
Fortunately, we can show a tight bound on the approximation
when the finite $\calW_\delta$ is chosen wisely.%
\begin{theorem} \label{thm:renyi_bound}
Let $\calW_\delta$ be the set of all wordforms with a probability of at least $\delta$, i.e. 
$\calW_\delta = \{\bw \mid \bw \in \calW, p(\bw) \ge \delta\}$.
We can bound our estimate error as:\footnote{We note that, in practice, we do not know the exact distribution $p(\bw)$ and use an estimate instead (detailed in \cref{sec:operationalisation}). 
\noindent This theorem only bounds one of the potential sources of uncertainty in our measurements: namely, metric computation, as opposed to model estimation.
}\looseness=-1%
\begin{align}\label{eq:bound}
    \ent_2(p) &\leq \hat{\ent}_2(p) \leq \ent_2(p) + \log\left(1 + \frac{(1 - \xi)\,\delta}{\eta}\right) \nonumber
\end{align}
where we can precisely compute both $\xi$ and $\eta$, which are defined as $\xi = \sum_{\bw \in \calW_\delta} p(\bw)$ and $\eta = \sum_{\bw \in \calW_\delta} p(\bw)^2$.
\end{theorem}

\begin{proof}
Proof is given in \cref{app:renyi_bound_proof}.
\end{proof}
\noindent This theorem implies $\hat \ent_2(p)$ is an upper bound on the true value $\ent_2(p)$, which can be made arbitrarily tight for small $\delta$ (we choose $\delta=10^{-8}$ here).

\subsection{A Null Hypothesis Test}

We now construct a null-hypothesis test to evaluate whether the observed lexicon is shaped by pressures in favour or against homophony.
Our \defn{null distribution} over lexica of size $M$ is defined as \begin{equation}
    p(\bW) = \prod_{m=1}^M p(\bw^{(m)})
\end{equation}
where $p(\bw)$ is a phonotactic distribution.
We further define a second distribution over values of the sample \renyi{} entropy, i.e. $p(\samplerenyi(\bW))$, where $\bW$ is distributed according to $p(\bW)$.
We can now ask whether the \renyi{} entropy in the observed lexicon is abnormal under the null distribution.
This suggests the following null hypothesis test:
\begin{itemize}
    \itemsep0em 
    \item $T_0$: $\samplerenyi (\widetilde{\bW})$ is sampled from $p(\samplerenyi(\bW))$
    \item $T_1$: $\samplerenyi (\widetilde{\bW})$ is not sampled from  $p(\samplerenyi(\bW))$
\end{itemize}
For a given $p(\bW)$, we can now test this hypothesis by evaluating the following probabilities:
\begin{equation}
    p(\samplerenyi (\bW) \leq \samplerenyi (\widetilde{\bW})) \text{ or } p(\samplerenyi (\bW) \geq \samplerenyi (\widetilde{\bW}))
\end{equation}
which we can estimate using Monte Carlo sampling. 
We reject the null hypothesis if either probability is smaller than $0{.}005$, which yields a confidence value of $p<0{.}01$ under a two-tailed test.
Strictly speaking, rejecting $T_0$, means that we have rejected that the sample \renyi{} entropy of the observed lexicon is plausibly consistent with the sample \renyi{} entropy of a lexicon sampled according to the null distribution $p(\bW)$.

We now analyse the assumptions we make by using $p(\bW)$ and discuss what conclusions we may be able to draw despite those assumptions. 
The two important assumptions are as follows:
\vspace{-2pt}
\begin{enumerate}[(i)]
    \itemsep0em 
    \item wordforms are sampled according to $p(\bw)$;
    \item wordforms are sampled i.i.d.
\end{enumerate}
\vspace{-2pt}
Therefore, if we believe our phonotactic distribution is correct---i.e., assumption (i) is good---this null hypothesis directly tests whether wordforms are sampled i.i.d.
Rejecting it, thus, gives us evidence that homophony is either favoured or hindered in a lexicon.
Should we believe assumption (i), we find evidence in support of homophony avoidance if the observed lexicon's sample \renyi{} entropy is significantly larger than the artificial one's. 
On the other hand, we find evidence of a pressure in favour of homophony if the observed lexicon's sample \renyi{} entropy is smaller than its artificial counterpart.
Assumption (i) is rather important, however, as we discuss further in \cref{sec:results}.

\section{Experimental Methodology} \label{sec:operationalisation}

The sample \renyi{} entropy, presented in \cref{eq:collision_lexicon}, can be directly computed on an observed lexicon.
On the other hand, both the \renyi{} entropy (as depicted in \cref{sec:renyi_entropy}) and our null hypothesis test are computed over a phonotactic distribution, to which we do not have direct access.
An important consideration, thus, is how exactly this distribution can be approximated.
Recently, \citet{trott2020human} relied on weakly regularised $n$-gram models for their analysis. 
As can be seen in our earlier results \citep{pimentel-etal-2020-phonotactic}, neural language models can capture this phonotactic distribution much more faithfully.
In this work, we will compare \citeauthor{trott2020human}'s $n$-grams with \citeauthor{pimentel-etal-2020-phonotactic}'s LSTM models, and show how $n$-grams may give misleading results.\looseness=-1

\vspace{-3pt}
\paragraph{$n$-gram.}
Perhaps the simplest method for estimating distributions of phones in a language is through $n$-gram, or in this case $n$-phone, modelling. 
Specifically, we can estimate the probability of observing some phone $w_t$ given the previous $n-1$ phones by computing the proportion of times this phone follows those previous $n-1$ phones in a corpus.
By this definition, sequences not present in the corpus will be assigned $0$ probability under the model. This, among other factors, contributes to the often poor generalisation abilities of basic $n$-gram models. Indeed, there exists an entire literature on smoothing and regularisation techniques for $n$-gram modelling \cite{katz,Ney94,chen-goodman-1996-empirical}. 
Laplace smoothing is a popular choice, being used in a number of recent works in computational linguistics \citep[e.g.][]{dautriche2017words,trott2020human}. However, it is perhaps the simplest of such regularisation techniques, and usually leads to much weaker empirical performances than, e.g., Kneyser--Ney \citep{Ney94}.
It is therefore natural to question whether an $n$-gram model with simple Laplace smoothing can provide a good representation of the true phonotactic distribution of a language.
In our experiments, we follow \citeauthor{trott2020human} in using a 5-gram model with Laplace smoothing with strength $0.01$ as $p(w_t \mid \bw_{<t})$.\looseness=-1

\vspace{-3pt}
\paragraph{LSTM.} 
In the task of sentence-level language modelling, neural models have surpassed their $n$-gram counterparts with respect to standard evaluation metrics. Neural architectures similarly outperform an $n$-phone model on the task of representing the phonotactic distribution.
We thus make use of a vanilla LSTM character-level language model to estimate this distribution, using a similar architecture to \citeposs{pimentel-etal-2020-phonotactic}. 
In short, we first retrieve a lookup embedding $\bz_t \in \mathbb{R}^e$ for each phone $w_t$ in a wordform.
We then feed these into an LSTM \citep{hochreiter1997long} to get hidden states $\bh_{t} \in \mathbb{R}^d$.
Finally, these hidden states are linearly transformed and processed by a softmax to arrive at a distribution $p(w_t \mid \bw_{<t})$ over the next token.
We train this model by minimising its cross-entropy with the distribution of the observed data. 
We use an LSTM architecture with 2 layers, an embedding size of 64, a hidden size of 256, and dropout probability of .33. This model is implemented using PyTorch \citep{pytorch2019} and optimised using Adam~\citep{kingma2014adam}.\looseness=-1

\paragraph{Model Selection.}
We evaluate the quality of our models by measuring their cross-entropy on held-out data, as is common in language modelling.
We report their train and test cross-entropies in \cref{tab:results_monomorphemic}.
Note that minimising this cross-entropy is equivalent to minimising the Kullback--Leibler divergence between our estimated model and the actual phonotactic distribution.
Thus, this serves as a metric of how well our model fits the data. 

\paragraph{Data.}
We use CELEX \citep{celex} as the source of data for our experiments, a dataset which covers three languages (English, German, and Dutch).
We restrict our analysis to mono-morphemic words,\footnote{We exclude words with spaces, hyphens, or apostrophes.}  
and note that we count words with multiple parts of speech as homophones (as both \citeauthor{piantadosi2012communicative} and \citeauthor{trott2020human} do).
This may inflate the number of homophones in our actual lexicons, thus reducing their surprisal in our analysis.
CELEX, however, marks zero derivation forms; we thus do not use these words on our analysis.
When computing the plug-in estimate of the lexicon's \renyi{} entropy
(in \cref{eq:collision_lexicon}) we use our entire dataset. 
We further use these wordforms to train our phonotactic models, splitting them in 80-10-10 train-validation-test sets. The test set is held out and only used for estimating the cross-entropy.
\section{Results,  Discussion and Conclusion\footnote{Our code is available at \url{https://github.com/rycolab/homophony-as-renyi-entropy}.}} \label{sec:results}

\Cref{tab:results_monomorphemic} displays our main results:\footnote{$\samplerenyi(\bW)$ is computed as the mean of our Monte Carlo samples for the phonotactic models.}
first we note that in terms of cross-entropy, the LSTM models provide better representations of the phonotactic distributions of all three languages.
Second, the Shannon entropy of the LSTM is smaller than the entropy of the $n$-gram. 
The $n$-gram, thus, appears to distribute probability mass more uniformly over the set $\calW$ than the LSTM, while the LSTM is more focused on the set of plausible wordforms.\looseness=-1

\begin{table}
    \centering
    \small
    \resizebox{\columnwidth}{!}{
    \begin{tabular}{llccccc}
    \toprule
        && \multicolumn{2}{c}{Cross-entropy} \\
        \cmidrule(lr){3-4}
        && Train & Test & $\ent_1(p)$ & $\ent_2(p)$ & $\samplerenyi(\bW)$ \\
         \midrule
        \\[-.9em]
        \multirow{3}{*}{\rotatebox{90}{\bf English}}
        &$n$-gram & 13.61 & 28.10 & 30.45 & 13.89 & 13.90$^*$ \\
        &LSTM & 18.75 & \textbf{19.89} & 26.46 & 14.77 & \textbf{14.77}$^*$ \\
        &Lexicon & - & - & - & - & 15.02\phantom{$^*$} \\
        \\[-.9em]
        \midrule
        \\[-.9em]
        \multirow{3}{*}{\rotatebox{90}{\bf German}}
        &$n$-gram & 14.08 & 29.25 & 30.44 & 14.27 & 14.27$^*$ \\
        &LSTM & 20.26 & \textbf{21.35} & 27.74 & 15.87 & \textbf{15.88}\phantom{$^*$} \\
        &Lexicon & - & - & - & - & \textbf{15.67}\phantom{$^*$} \\
        \\[-.9em]
        \midrule
        \\[-.9em]
        \multirow{3}{*}{\rotatebox{90}{\bf Dutch}}
        & $n$-gram & 13.89 & 26.08 & 30.45 & 14.07 & 14.06$^*$ \\
        & LSTM & 18.37 & \textbf{18.94} & 26.81 & 15.16 & 15.16$^*$ \\
        & Lexicon & - & - & - & - & \textbf{14.60}\phantom{$^*$} \\
        \\[-.9em]
         \bottomrule
         \multicolumn{7}{l}{$^*$Statistically different from lexicon's \renyi{} entropy ($p<0.01$).}
    \end{tabular}%
    }
    \vspace{-5pt}
    \caption{Cross-entropy, Shannon's entropy and \renyi{} entropy for both the $n$-gram, LSTM and lexicon.
    }
    \label{tab:results_monomorphemic}
    \vspace{-10pt}
\end{table}

The \renyi{} collision entropy results must be more carefully analysed.
At first glance, we see that the $n$-gram model has the smallest \renyi{} entropy across all languages, having more than 1 bit difference to the lexicon's sample \renyi{} entropy in both English and German.
This may lead one to conclude that homophony is strongly disfavoured in all these languages.
Nonetheless, the LSTM's collision entropy is considerably larger than the $n$-gram model's, while having both a lower cross-entropy and Shannon entropy.
We posit this is due to the $n$-gram strongly overfitting  the training set, giving these instances a higher probability than they are due.
These few overfit wordforms drive its \renyi{} entropy down, while the rest of the probability mass is spread over $\calW$ and increases the $n$-gram's Shannon entropy.\footnote{See \cref{app:shannon_vs_renyi} for a longer discussion of this behaviour.}
In other words, $n$-gram models do not approximate $p(\bw)$ well, and the assumption (i) of our hypothesis test does not hold.

When we compare the \renyi{} entropy of the LSTM to the lexicon's, we get much more nuanced results. 
While the English lexicon seems to hinder homophony---homophony is more surprising in real lexicons than expected from their phonotactics---the opposite is true for Dutch. 
Meanwhile, German presents no clear trends.
We should, however, refrain from making strong claims about these results.
While the difference between the LSTM's train and test cross-entropy is small, implying that it overfits only to a small degree, its precise quantitative impact on the \renyi{} entropy is hard to quantify.
Furthermore, expanding our analysis to CELEX's multi-morphemic words leads to somewhat different results (see \cref{app:multimorphemic}).
Hence, we see no clear pattern across these languages, and find, thus, no pressure either in favour or against homophony.\looseness=-1

We conclude this section with a warning.
When exploring linguistics using language models, one should carefully consider these model's inherent inductive biases and their potential effects on results.
While overfit $n$-grams provide strong evidence towards homophony avoidance in natural lexicons, we arrive at different results using better models.\looseness=-1%

\section*{Acknowledgements}

We thank the anonymous reviewers for their helpful feedback.  
We also thank Kyle Mahowald for numerous discussions about polysemy and homophony, M\'{a}rio Alvim for discussions about the \renyi{} entropy, and Sean Trott for detailed feedback on his paper and on this manuscript.

\section*{Ethical Considerations}

The authors foresee no ethical concerns with the research presented in this paper.

\bibliography{acl2020}
\bibliographystyle{acl_natbib}

\newpage
\appendix

\section{More Related Work} \label{app:related_work}

\begin{figure*}
    \centering
    \begin{subfigure}[b]{.5\textwidth}
        \includegraphics[width=\columnwidth]{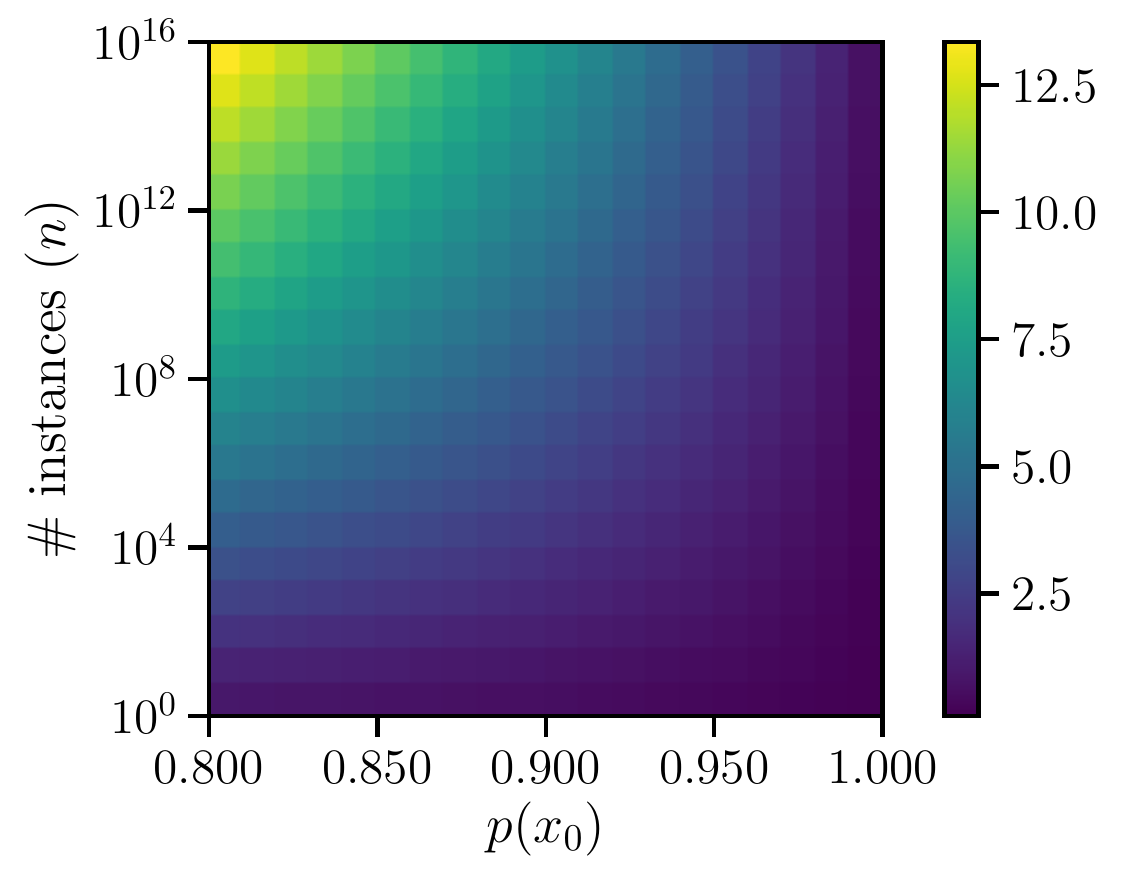}
        \caption{Shannon entropy}
    \end{subfigure}%
    ~
    \begin{subfigure}[b]{.5\textwidth}
        \includegraphics[width=\columnwidth]{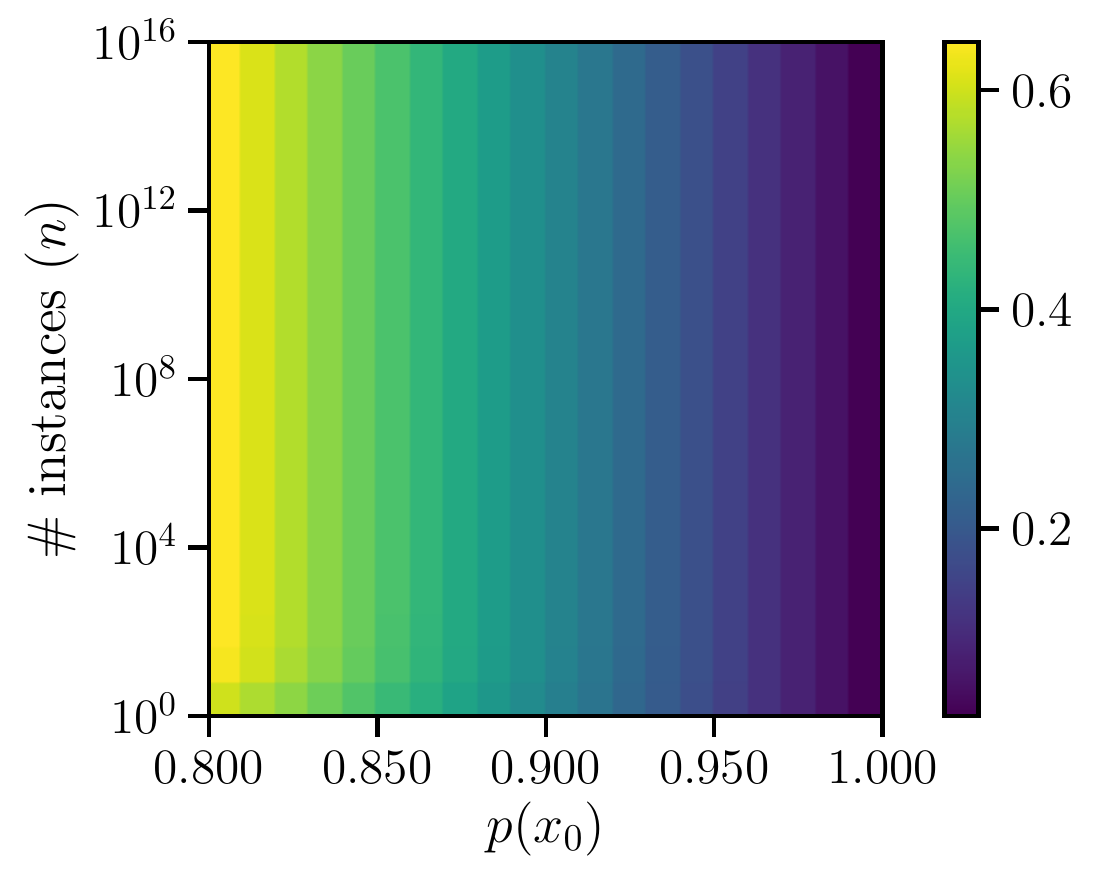}
        \caption{\renyi{} entropy}
    \end{subfigure}
    \caption{Shannon and \renyi{} collision entropies for a distribution where one instance has probability $p(x_0)$ and another $n$ instances have probability mass uniformly distributed among them.}
    \label{fig:shannon_renyi}
\end{figure*}

Related to the topic of homophony is the debate surrounding the effect of lexical neighbourhoods in the lexicon---of which homophony can potentially be seen as an extreme case.
While dense lexical neighbourhoods may hinder lexical recognition \citep{luce1986neighborhoods,luce1998recognizing,magnuson2007dynamics}, a dense phonotactic space would imply more economic wordforms \citep{zipf1949human,piantadosi2011word}. 
Furthermore, phonotactically well-formed words are both recognised faster \citep{vitevitch1999phonotactics} and easier for young children to learn  \citep{coady2004young}.
It is thus not clear if denser lexical neighbourhoods would lead to a more or less ``efficient'' lexicon.
\citet{dautriche2017words} showed, using a very similar methodology to \citeposs{trott2020human}, that words have more phonological neighbours than would be expected by chance.
Their results, though, are vulnerable to similar criticisms to the ones that we provide in this paper about the use of potentially overfit $n$-gram models.\looseness=-1

\citet{caplan2020miller} propose a phonotactic monkey---an extension of \citeposs{miller1957some} random typing thought experiment---to make a similar point to \citeauthor{trott2020human}.
They rely on similar $n$-gram phonotactic models for their analysis of homophony, being thus vulnerable to the same criticisms we present here.
\citeauthor{caplan2020miller}, however, refrain from making claims about a pressure towards or against homophony---positing homophony is only a result of chance, a byproduct of language's generation process.
\citeauthor{caplan2020miller} further analyse polysemy in their paper, a topic with which we do not engage here.
Incorporating form--meaning interactions to this analysis is not straightforward, but would be interesting as future work.
We believe, however, that when studying homophony, ignoring them is not particularly problematic, since the multiple meanings a wordform takes are unrelated by definition.
Nonetheless, it would be critical to incorporate meaning in our model if we wanted to study other forms of lexical ambiguity.
Furthermore, several recent work has shown that the lexicon is (weakly) systematic, i.e. words with similar meanings are more likely to have similar wordforms \citep{dautriche2017wordform,gutierrez-etal-2016-finding,pimentel-etal-2019-meaning}.
Our i.i.d. use of a phonotactic distribution, though, completely ignores these form--meaning correlations.

\hfill
\newpage

\section{Multi-morphemic results} \label{app:multimorphemic}

\begin{table}[h]
    \centering
    \small
    \resizebox{\columnwidth}{!}{
    \begin{tabular}{llccccc}
    \toprule
        && \multicolumn{2}{c}{Cross-entropy} \\
        \cmidrule(lr){3-4}
        && Train & Test & $\ent_1(p)$ & $\ent_2(p)$ & $\samplerenyi(\bW)$ \\
         \midrule
        \\[-.7em]
        \multirow{3}{*}{\rotatebox{90}{\bf English}}
        &$n$-gram & 17.97 & 27.16 & 29.74 & 16.07 & 16.07$^*$ \\
        &LSTM & 25.00 & \textbf{25.55} & 28.92 & 17.97 & 17.96$^*$  \\ 
        &Lexicon & - & - & - & - & \textbf{16.67}\phantom{$^*$} \\
        \\[-.7em]
        \midrule
        \\[-.7em]
        \multirow{3}{*}{\rotatebox{90}{\bf German}}
        &$n$-gram & 20.37 & 29.67 & 31.39 & 15.69 & 15.69$^*$ \\
        &LSTM & 27.51 & \textbf{27.94} & 31.55 & 18.85 & \textbf{18.85}$^*$  \\ 
        &Lexicon & - & - & - & - & 19.91\phantom{$^*$} \\ 
        \\[-.7em]
        \midrule
        \\[-.7em]
        \multirow{3}{*}{\rotatebox{90}{\bf Dutch}}
        & $n$-gram & 23.03 & 30.61 & 31.81 & 15.37 & 15.37$^*$ \\
        &LSTM & 29.37 & \textbf{29.76} & 30.99 & 19.54 & \textbf{19.54}$^*$  \\ 
        &Lexicon & - & - & - & - & 20.57\phantom{$^*$} \\
        \\[-.7em]
         \bottomrule
         \multicolumn{7}{l}{$^*$Statistically different from lexicon's \renyi{} ($p<0.01$).}
    \end{tabular}%
    }
    \caption{Cross-entropy, Shannon's entropy and \renyi{} entropy for both the $n$-gram, LSTM and real lexicon while also analysing multi-morphemic wordforms. 
    }
    \label{tab:results}
    \vspace{-10pt}
\end{table}

\begin{figure}[t]
    \centering
    \includegraphics[width=.99\columnwidth]{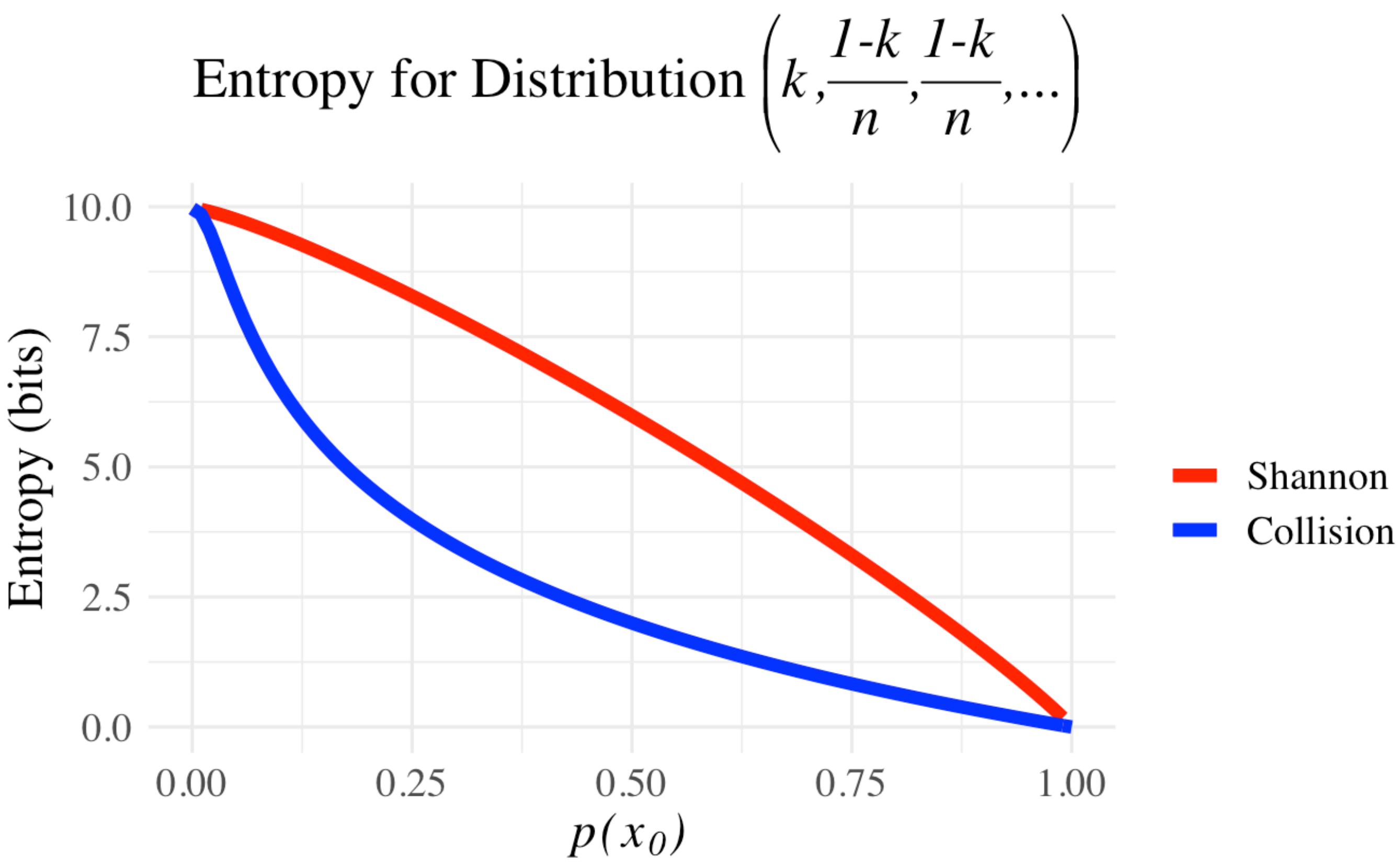}
    \caption{Shannon and \renyi{} collision entropies for distribution with uniform mass spread over all but one instance.}
    \label{fig:shannon_renyi_fixedn}
\end{figure}

\section{Shannon vs. Collision Entropy} \label{app:shannon_vs_renyi}

In this section, we exemplify the difference between the \renyi{} and Shannon entropies.
With that in mind, we define a distribution over $n+1$ instances, where probability mass is distributed such that:\looseness=-1
\begin{align}
    p(x) = \left\{ \begin{array}{cc}
        k & x=x_0 \\
        \frac{1 - k}{n} & x \neq x_0
    \end{array} \right.
\end{align}
This distribution, thus, puts $k$ probability mass on $x_0$, and uniformly distributes the rest among the $n$ other instances.

\Cref{fig:shannon_renyi_fixedn} shows the behaviour of both entropies with $n$ fixed at $99$ and while we vary the mass in $p(x_0)$. In it, we see how the \renyi{} entropy is always smaller or equal to the Shannon entropy---an important property of the \renyi{} collision entropy.

\Cref{fig:shannon_renyi} presents these entropies while changing both $n$ and $k$, i.e. $p(x_0)$.
In this figure, we see that when some large probability mass is already allocated to a single instance (or a few), the \renyi{} entropy becomes relatively constant with relation to the distribution among the other instances.
The Shannon entropy, on the other hand, is still susceptible to these other instances distribution, and goes to infinity as $n \rightarrow \infty$.

Relating this to our analysed $n$-gram models, we see that, by allocating a large probability mass to the training set, they can obtain a small \renyi{} entropy. However, since they smoothly distribute the rest of their probability mass throughout $\calW$ they achieve a high Shannon entropy.

\onecolumn

\section{Proof of \cref{thm:renyi_bound}} \label{app:renyi_bound_proof}

\textbf{\Cref{thm:renyi_bound}.}
\textit{
Let $\calW_\delta$ be the set of all wordforms with a probability at least $\delta$, i.e.
\begin{equation}
    \calW_\delta = \{\bw \mid \bw \in \calW, p(\bw) \ge \delta  \} \nonumber
\end{equation}
We can bound our estimate error as:
\begin{align}
    \ent_2(p) &\leq \hat{\ent}_2(p) \leq \ent_2(p) + \log\left(1 + \frac{(1 - \xi)\,\delta}{\eta}\right) \nonumber
\end{align}
where we can precisely compute both $\xi$ and $\eta$, which are defined as
\begin{align}
    \xi = \sum_{\bw \in \calW_\delta} p(\bw), \qquad \eta = \sum_{\bw \in \calW_\delta} p(\bw)^2 \nonumber
\end{align}
}

\begin{proof}
We first decompose the error in our estimate as
\begin{subequations}
\begin{align}
    \hat{\ent}_2(p) - \ent_2(p) & \stackrel{(1)}{=} \log\left(\sum_{\bw \in \calW_\delta} p(\bw)^2 + \sum_{\bw \in \calW \setminus\calW_\delta} p(\bw)^2 \right) -\log\left(\sum_{\bw \in \calW_\delta} p(\bw)^2 \right)  \\
    &=\log\left(\frac{\sum_{\bw \in \calW_\delta} p(\bw)^2 + \sum_{\bw \in \calW \setminus\calW_\delta} p(\bw)^2}{\sum_{\bw \in \calW_\delta} p(\bw)^2}\right)  \\
    &= \log\left(1 + \frac{\sum_{\bw \in \calW \setminus\calW_\delta} p(\bw)^2}{\eta}\right) 
\end{align}
\end{subequations}
where $\eta = \sum_{\bw \in \calW_\delta} p(\bw)^2$ and equality (1) follows from the definition of $\ent_2$ and the separation of the sum into two parts.
We define
$\xi=\sum_{\bw \in \calW_\delta} p(\bw)$ and, thus, $1 -\xi=\sum_{\bw \in \calW \setminus \calW_\delta} p(\bw)$.
Now, by invoking \Cref{lemma:technical}, we have the following inequality
\begin{align}
    0 \leq \sum_{\bw \in \calW \setminus\calW_\delta} p(\bw)^2 &\leq \left( 1 - \xi \right) \delta 
\end{align}
which proves the theorem.

\end{proof}

\begin{lemma}{(Technical Lemma)}\label{lemma:technical}
Let $\{x_n\}_{n=1}^N$ be real values in the interval $[0, \delta]$ such that $\sum_{n=1}^N x_n = \beta$.
Then, 
\begin{equation} \label{eq:technical_lemma}
    \sum_{n=1}^N x_n^2 \leq \beta \cdot \delta 
\end{equation}
\end{lemma}
\begin{proof}
We claim the maximal solution---i.e., the set $\{x_n\}_{n=1}^N$ which maximises \cref{eq:technical_lemma}---is $x_k = \delta$ for $k \in 1, \ldots, K$ and $x_{K+1} = \beta - K \delta$ for some $K < N$.
We prove its maximality by contradiction.
Suppose there exists another maximal solution.
Further, suppose that the values for this solution are sorted, such that $x_i \ge x_j$ for any $i<j$.
Then, there must exist two indices $i$ and $j$ such that $\delta > x_i \geq x_j$ and $i<j$. 
Now, let $\epsilon = \delta - x_i > 0$. We can prove that:\looseness=-1%
\begin{subequations}
\begin{align}
     (x_i + \epsilon)^2 + (x_j - \epsilon)^2 
     &= x_i^2 + 2 x_i \epsilon + \epsilon^2 + x_j^2 - 2 x_j \epsilon + \epsilon^2\\
     &= x_i^2 + x_j^2 + 2 \underbrace{(x_i - x_j)}_{\ge 0} \epsilon + 2 \epsilon^2 \\
     &\stackrel{(1)}{>} x_i^2 + x_j^2
\end{align}
\end{subequations}
where (1) relies on the fact that, per our assumptions, $\epsilon > 0$ and $x_i \ge x_j$.
Since this is a strict inequality, this alternate solution is not truly maximal, completing our proof by contradiction.

Now, we can use our maximal solution to compute the desired upper-bound.
First, we note the value of $K=\left\lfloor \frac{\beta}{\delta} \right\rfloor$.
The value for the maximal solution will thus be
\begin{subequations}
\begin{align}
    \sum_{n=1}^N x_n^2 &= K \cdot \delta^2 +  \left(\beta - K \delta \right)^2 \\
    &= \left\lfloor \frac{\beta}{\delta} \right\rfloor \cdot \delta^2 + \left(\beta - \left\lfloor \frac{\beta}{\delta} \right\rfloor \delta \right)^2 \\
    &= \left\lfloor \frac{\beta}{\delta} \right\rfloor \cdot \delta^2 + \underbrace{\left(\beta - \left\lfloor \frac{\beta}{\delta} \right\rfloor \delta \right)}_{\ge 0} \underbrace{\left(\beta - \left\lfloor \frac{\beta}{\delta} \right\rfloor \delta \right)}_{< \delta} \\
    &\le \left\lfloor \frac{\beta}{\delta} \right\rfloor \cdot \delta^2 + \left(\beta - \left\lfloor \frac{\beta}{\delta} \right\rfloor \delta \right) \cdot \delta \\
    &= \left\lfloor \frac{\beta}{\delta} \right\rfloor \cdot \delta^2 + \left(\frac{\beta}{\delta}\delta - \left\lfloor \frac{\beta}{\delta} \right\rfloor \delta \right) \cdot \delta \\
    &= \left\lfloor \frac{\beta}{\delta} \right\rfloor \cdot \delta^2 + \left(\frac{\beta}{\delta} - \left\lfloor \frac{\beta}{\delta} \right\rfloor \right) \cdot \delta^2 \\
    &=\beta \cdot \delta
\end{align}
\end{subequations}
As the value of the maximal solution is bounded by $\beta \cdot \delta$, any other solution must also have a value smaller than this, which completes the proof.

\end{proof}

\end{document}